\documentclass[twocolumn,twoside]{IEEEtran}

\usepackage{amsmath,amstext,amsthm,amssymb,epsf}
\usepackage{latexsym}
\usepackage[pdftex]{graphicx}
\usepackage{setspace}
\usepackage{float}
\usepackage{placeins}
\usepackage{hyperref}
\usepackage{caption} 
\usepackage{xcolor}
\numberwithin{equation}{section} 

\newtheorem{theorem}{\sc Theorem}

\newtheorem{lemma}{\sc Lemma}
\newtheorem{coro}{\sc Corollary}
\newtheorem{req}{\sc Requirement}

\newtheorem{defin}{\sc Definition}
\newtheorem{rem}{\sc Remark}
\newtheorem{cla}{\sc Claim}
\newtheorem{ex}{\sc Example}

\newenvironment{remark}{\begin{rem}}{\hspace*{\fill}$\Diamond$\end{rem}}
\newenvironment{example}{\begin{ex}}{\hspace*{\fill}$\diamondsuit$\end{ex}}
\newenvironment{claim}{\begin{cla}}{\end{cla}}

\newenvironment{definition}{\begin{defin}}{\end{defin}}
\renewcommand{\emptyset}{\varnothing}

\begin{document}

\title{The cluster structure function}
\author{Andrew R. Cohen\thanks
{Andrew Cohen is with the Department of Electrical and Computer Engineering,
Drexel University.
Address: A.R. Cohen, 3120--40 Market Street,
Suite 313, Philadelphia, PA 19104,
USA. Email: {\tt andrew.r.cohen@drexel.edu}}
and Paul M.B. Vit\'{a}nyi
\thanks{
Paul Vit\'{a}nyi is with the National Research Center for Mathematics
and Computer Science in the Netherlands (CWI),
and the University of Amsterdam.
Address:
CWI, Science Park 123,
1098XG Amsterdam, The Netherlands.
Email: {\tt Paul.Vitanyi@cwi.nl}.
}}

\maketitle
\begin{abstract}

For each partition of a data set into a given number of 
parts there is a partition
such that every part is as much as possible a good model
(an ``algorithmic sufficient statistic'') for the data in that part.
Since this can be done for every number between one and
the number of data, the result is a function,
{\em the cluster structure function}. It maps the number of
parts of a partition to values related to the deficiencies of 
being good models by the parts. Such a function starts
with a value at least zero for no partition of the data set
and descents to zero for the partition of the data set
into singleton parts. The optimal clustering is the one selected by analyzing 
the cluster structure function. The theory behind the method is expressed in algorithmic information 
theory (Kolmogorov complexity). In practice
the Kolmogorov complexities involved are approximated by 
a concrete compressor. We give examples using real data sets:
the MNIST handwritten digits and the segmentation of real cells 
as used in stem cell research.

\end{abstract}

\section{Introduction}

The aim of this work is to introduce the cluster structure function and apply it to
propose a method for finding the number of clusters in
a given dataset that is unsupervised, feasible, justifiable an terms of its theory, and
more accurate than previous methods for this task. \color{black}
Clustering is a fundamental task in unsupervised learning, partitioning a set of objects into groups called clusters such that  
objects in the same cluster are more similar to each other than to those in other groups
\cite{TK08}. Every object in a computer is represented by a finite sequence 
of 0's and 1's: a
finite binary string (abbreviated to ``string'' in the sequel).
There are many methods and algorithms for clustering and determining 
the number of clusters in data as for example
surveyed in \cite{An14,JD88,KR09,TK08}. We explore a new method for determining
the number of clusters  based on
Kolmogorov's notion of algorithmic sufficient statistic
\cite{Sh83,CT91} which is expressed in terms of 
Kolmogorov complexity \cite{Ko65}. For technical reasons we use
{\em prefix Kolmogorov complexity} \cite{Le74}.
In the sequel we also use $K$ for the number of clusters in the data, agreeing
with customary use. Confusion is avoided by the context. 
 
A brief overview of the needed notions is given here.
Details and proofs can be found in the textbook \cite{LV08}.
A prefix Turing machine is a Turing machine (we use a binary alphabet)  
such that the set of input programs for which the machine
halts is a prefix code (no input program is a proper prefix of another one). The prefix Turing
machines can be computationally enumerated $T_1,T_2, \ldots$ and this list has a universal
prefix Turing machine   $U$ such that $U(i,p)=T_i(p)$ for all integers $i$
and halting programs $p$ for $T_i$.
Formally, the {\em conditional prefix Kolmogorov complexity}
$K(x|y)$ is the length of the shortest input string  $z$
such that the reference universal prefix Turing machine $U$ on input $z$ with
auxiliary information $y$ outputs $x$. The
{\em unconditional prefix  Kolmogorov complexity} $K(x)$ is defined as
$K(x|\epsilon)$ where $\epsilon$ is the empty string.
The quantity $K(x)$ is the length  of a shortest binary string $x^*$
from which $x$ can be effectively reconstructed. 
If there are more than one candidates for $x^*$  we use the first one in the enumeration.
The string $x^*$ accounts for every effective regularity in $x$.
In these definitions both $x$ and $y$ can consist of strings into which
finite multisets of finite binary strings are encoded. \color{black}

Informally, a finite set $A$ of strings containing $x$ is an
{\em algorithmic sufficient statistic} for $x$ iff $K(A)+\log |A|=K(x)$.
That is, the encoding of $x$ by giving $A$ (a model) and 
the index of $x$ in $A$ is as short as a shortest computer program for $x$
(sometimes one adds also a small value).
This means that $A$ is a good model for $x$ \cite{VV04}. 
As we show in Lemma~\ref{lem.sufficiency} it is impossible that  $A$ is such a good
model for all $y \in A$.
Therefore we have to relax the condition of sufficiency.
If the equality above holds up to some additive term then this term is called 
the {\em optimality deficiency}. 
We propose to group
the elements from a data set (a multiset) into clusters (submultisets) such
that the optimality deficiencies in every cluster are minimal 
in some sense.
This seems to require a specification of the number of clusters. 
However, the aim is
to find the number of clusters. To solve this conundrum the proposed method 
proceeds as follows.
The cluster structure function has
the number of clusters as argument and a quantity involving the
optimality deficiencies as value. Such a function
decreases to 0 when the number of clusters grows 
to the cardinality of the data set. The optimal number of clusters can then 
be selected related to the cluster structure function.

We give the definitions and the ideal method of application in 
Section~\ref{sect.csf}. Proofs are deferred to Section~\ref{sect.proofs}. An explanation of the probability  relations 
of members of a cluster is given in Section~\ref{sect.meaning}. A brief survey of related literature is given in Section~\ref{sect.rlit}. Finally, Section~\ref{sect.eappl} shows examples of real applications including estimating the number of unique digits in a set of MNIST handwritten digits and an ensemble segmentation approach to human stem cell nuclear segmentation.

\section{Theory of the cluster structure function}\label{sect.csf}

The aim is to partition a multiset into submultisets 
such that each submultiset
constitutes a cluster. In probabilistic statistics the relevant
notion is  the ``sufficient statistic'' due to R.J. Fisher \cite{Fi22,CT91}. 
According to Fisher:
\begin{quote}
       ``The statistic chosen should summarise the whole of the relevant
information supplied by the sample. This may be called
       the Criterion of Sufficiency $\ldots$
In the case of the normal curve
of distribution it is evident that the second moment is a
       sufficient statistic for estimating the standard deviation.''
\end{quote}
This type of sufficient statistic pertains to probability distributions.
In the problem at hand the data are individual strings. 
Therefore the probabilistic notion is not appropriate. 
For individual strings the analogous notion is the 
``algorithmic sufficient statistic.''
For convenience we delete the adjective "algorithmic" in the sequel 
(probabilistic sufficient statistic doesn't occur in the sequel).
We equate a multiset being a cluster with the multiset being, 
as close as possible according to a given criterion, an (algorithmic) 
sufficient statistic for the members of the cluster. 
The new method partitions $S$ such that
each resulting part is as close as possible to the given criterion 
a sufficient statistic for all of its members. Therefore they are good models 
for its members \cite{VV04}. This is different from existing methods which use 
some metric which does not say much about this aspect. 

\begin{definition}
\rm
A multiset $A$  of strings is an {\em algorithmic sufficient statistic}
abbreviated as {\em sufficient statistic} for a element $x \in A$
if $K(A)+\log |A|=K(x)$. 
\end{definition}
Here $A$ is a model and the 
$\log |A|$ term allows us to pinpoint $x$ in $A$.
Therefore, every $y \in A$ satisfies 
$K(A)+\log |A| \geq K(y)$. Reference \cite{GTV01} tells us 
that if $A$ is a sufficient statistic for the string $x$ then $K(A|x)=O(1)$. That is,
$A$ is almost completely determined by $x$.
If $A$ is a sufficient statistic for $x$, then
$K(x|A) = \log |A|$. Namely, 
$K(x) \leq K(A)+K(x|A) 
\leq K(A)+\log |A|=K(x)$. We call  $x$ a 
{\em typical} member of $A$. 

This is akin to the minimum description length (MDL) principle 
in Statistics \cite{mdl}.
To illustrate, if the length of a binary string  $x$ is $n$ and  
$K(x)=n+K(n)$ (the maximum) which means
that $x$ is random then $A=\{x\}$ is a
sufficient statistic of $x$ (the minimal one) and $A=\{0,1\}^n$ is 
also a sufficient statistic of $x$ (the maximal one).
There is a tradeoff between the cardinality of a sufficient statistic  $A$ of a string $x$
and the amount of effective regularities in the
string $x$ it represents.  The greater the cardinality of $A$ is the smaller is 
$K(A)$ which is the amount of effective
regularities it represents.
The multiset $A$ accounts for as many effective regularities in $x$ as is possible for
a set of the cardinality of $A$.
This  means that $A$ is the model of best fit, which we call the best model,
for $x$ which is possible  \cite[Section IV-B]{VV04}.
Thus, if $A$ has the property that for every $y \in A$ it is as much as possible a sufficient statistic,
then all members of $A$ share as many effective regularities as is possible. All the $y \in A$ 
are similar in the sense of \cite{LV01,CV04}. We cluster the data
according to this criterium. \color{black}

If $A$ contains elements $y$ such that $K(A)+\log |A| > K(y)$ (trivially
$<$ is impossible) then $K(A|y) \neq O(1)$. Let us look closer at
what this implies and consider $A$ containing only elements
of length $n$. Then by the symmetry of information \cite{Ga74}
we have $K(A|x)=K(A)+K(x|A)-K(x)+O(\log n)$.
For example, let $A$ be the set containing
all integers in an interval with complex endpoints and $x$ 
an integer in this interval of low complexity.
For example $K(A) = \Omega(n)$ and $K(x)=o(n/4)$.
Therefore $K(x|A) =o(n/4)$ and this yields
$K(A|x) = \Omega(n)$. That is, $A$ is not at all determined by $x$. 
\begin{definition}\label{def.delta}
\rm
The {\em optimality deficiency} of $A$ as a 
{\em sufficient statistic} for $x \in A$ is
\begin{equation}\label{eq.delta}
\delta(A,x)= K(A)+ \log |A| -K(x). 
\end{equation}
The {\em mean} of the optimality deficiencies of a set $A$ is
\[
\mu_{A} = \frac{1}{|A|} \sum_{x \in A} \delta(A,x).
\]
\end{definition}
 
Here $\delta(A,x) \geq 0$ with equality for a proper sufficient statistic.
If $\mu_A=0$ then $\delta (A,x)=0$ for all $x \in A$, that is,
$A$ is a sufficient statistic for all of its elements. But this is not
possible for $|A| \geq 2$ by the following lemma.

\begin{lemma}\label{lem.sufficiency}
Let $A$ be a finite multiset of strings of length $n$.

{\rm (i)} Let $\delta(A,x)=0$ for some $x \in A$. 
For all $y \in A$ holds $K(y) \leq K(x)$ and if $|A|>2$ then
$K(y) < K(x)$ for some $y \in A$, $\delta(A,y) > 0$,
and $\mu_A > 0$.
 
{\rm (ii)} There exist $A$ and $x \in A$ such that
$\delta(A,x) < 0$ and for such $A$ no $y \in A$ satisfies $\delta(A,y)=0$.
\end{lemma}
\begin{remark}
\rm
The optimality deficiency should not be confused with
the {\em randomness deficiency} of 
$x \in A$ with respect to $A$:
\[
\delta(x|A)= \log |A| - K(x|A).
\]
By the symmetry of information law $K(A)+K(x|A)=K(x)+K(A|x)$ up to a 
logarithmic additive term $O( \log K(A))$. Therefore
$\delta(x|A)+K(A|x)= \log |A| +K(A)-K(x)+O(\log K(A))$ and hence
$\delta(A,x) = \delta(x|A)+K(A|x)+O(\log K(A))$.
\end{remark}

For clustering we want
ideally the model to be a sufficient statistic for all elements in it.
But we have to deal
with optimality deficiencies which are greater than 0, and 
with real data typically
they are all greater than 0. There are many ways to combine the optimality 
deficiencies (or other aspects) to obtain criteria for selection. This 
is formulated in the criterion function as follows.

Let ${\cal N}$ denote the natural numbers and
$S= \{x_1, \ldots , x_n\}$ be a finite nonempty multiset of strings. 
Consider a partition $\pi$ of $S$ into $k$ nonempty subsets
$S_1, \ldots , S_k$ such that $\bigcup_{i=1}^k S_i =S$ and
$S= \{x_1, \ldots , x_n\}$ be a finite nonempty multiset of strings. 
Consider a partition $\pi$ of $S$ into $k$ nonempty subsets
$S_1, \ldots , S_k$ such that $\bigcup_{i=1}^k S_i =S$ and
$S_i \bigcap S_j= \emptyset$ for $i \neq j$. Denote the set of partitions
of $S$ into $k$ submultisets by $\Pi_k$ and the set of all partitions by $\Pi$.
The {\em criterion function} $f: \Pi \rightarrow {\cal N}$ takes as argument
a partition $\pi \in \Pi$ of $S$ and as value a natural number 
computed from the optimality deficiencies 
involved in the partition subject to the following:
(i) the value of $f(\pi)$ does not increase if one or more
optimality deficiencies are changed to 0; and (ii) $f(\pi)=0$ if 
all optimality deficiencies are 0. (One can use other aspects as well.)
\begin{definition}\label{def.H}
\rm
The {\em Cluster Structure Function} (CSF) 
\footnote{The cluster structure function is named in analogy with 
the Kolmogorov structure function $h_x: {\cal N} \rightarrow {\cal N}$ 
defined by
$h_x(k)= \min_{S \subseteq \{0,1\}^*} \{\log |S|: x \in S, \; K(S) \leq k \}$
associated with a binary finite string $x$ \cite{VV04}.} 
for a multiset $S$ of $n$ strings is defined by 
\begin{equation}\label{eq.H}
H^{f}_S(k)= \min_{\pi \in \Pi_k} f(\pi)
\end{equation}
where $f$ is the criterion function, for each $k$ ($1 \leq k \leq n$).
The graph of this function is called the {\em CSF curve}.
If $f$ is understood we may write $H_S$ for the CSF function.
%

\end{definition}

\begin{example}
\rm
Let $\pi \in \Pi_k$. The {\em bandwidth} of $S_i$ is
$b_i =\max_{x\in S_i} \{\delta(S_i,x)\} - \min_{x\in S_i} \{\delta(S_i,x)\}$.
Define $f(\pi) = \min \sum_{1 \leq i \leq k} b_i$.
For every $k$ ($1 \leq k \leq n$) the value $H^{f}_S(k)$ 
is based on the partition $\pi \in \Pi_k$ that minimizes the minimal sum of 
of the bandwidths of the parts in a $k$-partition of $S$. 
If we consider the graph of $H^{f}_S$ in a two-dimensional 
plane with the horizontal axis
denoting the number $k$ of parts of $S$, and the vertical
axis denoting the value of $H^{f}_S$, then left of the graph of 
$H^{f}_S$ there are no possible
$k$-partitions while right of the graph of $H^{f}_S$ there are redundant 
$k$-partitions. On the graph of $H^{f}_S$ occur the witness partitions.
\end{example}

\begin{remark}
\rm
By Lemma~\ref{lem.sufficiency} 
parts $A$ with $|A| >2$ of a witness partition $\pi$ of $S$ 
can not be a sufficient statistic for all of its elements and 
therefore $f(\pi) >0$. 
\end{remark}
\begin{remark}
\rm
In clusters the
members of a cluster typically share some characteristics but not 
all characteristics. It turns out that the members of a cluster are 
probabilistically close, Section~\ref{sect.meaning}. 
\end{remark}

\subsection{Properties}
It is convenient to ignore possible $O(1)$ additive terms in the sequel.
\begin{lemma}\label{lemma.1}
Let $S =\{x_1, \ldots , x_n \}$ with $n \geq 1$.
For every $f$ we have $H^{f}_S(n)=0$
and $H^{f}_S$ is monotonic non-increasing with increasing arguments 
on its domain $[1,n]$.
\end{lemma}

The graph of $H^f_S$ descends until $H_S^{f}(k)=0$ for the least $k \leq n$,
$H^{f}_S(k) = \cdots = H^{f}_S(n) =0$.
We give a lower bound on $H^{f}$ for some datasets $S$.
\begin{lemma}\label{lemma.lb}

There exist $S \subseteq \{0,1\}^m$ with $|S|=n$ and $n \leq m$ such that 
$H_S^{f}(k)=0$ for all $1 \leq k \leq n$ up to an additive term
of $O(\log K(S))$.
\end{lemma}

The following lemma establishes that there are sets $S$ of $n$ elements
such that $H_S^{f}$ stays at a high value for arguments 
$1, \ldots ,n-1$ and drops suddenly to 0 for argument $n$. 

\begin{lemma}\label{lemma.ub}
There exists a set $S \subseteq \{0,1\}^m$ and $|S|=n$
with $m$ a sufficiently large multiple of $n$ such that 
$H^{f}_S(k) \geq m/n$ for 
$1 \leq k \leq n-1$ and $H^{f}_S(n)=0$. 
\end{lemma}


In practice we may use the optimality deficiencies within the standard 
deviation around the mean to determine the criterion function $f(\pi)$
for a partition 
$\pi \in \Pi_k$ of $S$ into parts $S_1, \ldots , S_k$.
This is a more refined method since it eliminates the outliers.
only counting the central items (68.2\% if they are normally distributed)
of the optimality deficiencies in each part $S_i$. The {\em mean} of $S$ is 
$\mu_S = 1/|S| \sum_{x \in S} x$.  
The {\em standard deviation} of the $\delta(S,x)$ of a multiset $S$ is
\[
\sigma_S = \frac{1}{|S|} \sqrt{\sum_{x \in S} 
(\delta(S,x) - \mu_S)^2}. 
\]
\begin{definition}
\rm
Let $S$ be a multiset of strings, 
$S_{\sigma}= \{x \in S:|x-\mu_S| \leq \sigma_S\}$ and 
$f_{\sigma}$ is the criterion function of $S_{\sigma}$.
\begin{eqnarray}
H_{S_\sigma}^{f_{\sigma}}(k)&=& \min_{\pi \in \Pi_k} \max_{1 \leq i \leq k} 
f_{\sigma} (\pi) ,
\end{eqnarray}
where $\pi$ divides $S_{\sigma} = \bigcup_{1 \leq i \leq k} S_{\sigma,i}$ into 
$k$ parts $S_{\sigma,1}, \ldots , S_{\sigma,k}$
\end{definition}
That is, $H_{S_\sigma}^{f_{\sigma}}(k)$ is the minimum over all 
partitions of $S_{\sigma}$ into $k$ parts.
It clusters possibly better since  
$H_S^{f_{\sigma}}(k) \leq H_S^{f}(k)$ for all $k$ ($1 \leq k \leq n$) implying
by Section~\ref{sect.meaning} that the conditional 
probabilities between most members of a part of
a witness partition may be larger but never smaller 
using $H_S^{f_{\sigma}}(k)$ than using $H_S^{f}(k)$.
\begin{lemma}\label{lemma.4}
Let $S =\{x_1, \ldots , x_n \}$ with $n \geq 2$.
Then 
$H_{S_{\sigma}}^{f_{\sigma}}(1) >0 $, 
$H^{f_{\sigma}}_{S_{\sigma}}(n)=0$, and
$H^{f_{\sigma}}_{S_{\sigma}}$ is monotonic non-increasing.
\end{lemma}

\begin{lemma}\label{lemma.lbsd}
There exists $S \subseteq \{0,1\}^m$ with $|S|=n$ such that 
$H_{S_{\sigma}}^{f_{\sigma}}(k)=0$ for all $1 \leq k \leq n$ up to 
an additive term $O(\log K(S))$.
\end{lemma}

\begin{lemma}\label{lemma.ubsd}
There exists a multiset $S \subset \{0,1\}^m$ and $|S|=n$
with $m$ multiple of $n$ such that $H^{f_{\sigma}}_{S_{\sigma}}(k) \geq m/n$ for
$1 \leq k \leq n-1$ and $H^{f_{\sigma}}_{S_{\sigma}}(n)=0$.
\end{lemma}

\subsection{Proofs}\label{sect.proofs}
\begin{proof} of Lemma~\ref{lem.sufficiency}
(i) For all $y \in A$ we have $K(y) \leq K(A)+\log|A|$ which implies
$K(y) \leq K(x)$ (since $K(x)=K(A)+\log|A|$) and therefore 
$\delta(A,y)\geq 0$ and hence $\mu_A \geq 0$. 
For $|A| >2$ there are $y \in A$ such that $K(y)<K(x)$ since 
$K(y) < K(A)+\log |A|=K(x)$.
For example if $y$ is the first element of $A$ and therefore $K(y) \leq K(A)$.
Hence $\delta(A,y) > 0$ and $\mu_A >0$.

Ad (ii) There is an $x \in A$ such that $\delta(A,x) < 0$. 
For example $A$ is a sufficiently long 
interval of integers of (represented by $n$-bit strings) of length $O(2^n)$ 
with end points
of $O(\log n)$ Kolmogorov complexity and $x \in A$ is a random string
in that interval which means $K(x) = \Omega (n)$. Then $\delta(A,x) <0$ and 
by Item (i) there are no $y \in A$ such that $\delta(A,y)=0$. 
\end{proof}
 
\begin{proof} of Lemma~\ref{lemma.1} 
The graph of $H^{f}_S$ starts with the partition of $S$ into 1 part
(no partition). 

$n=1$. The optimality deficiency involved is 0
and by Definition~\ref{def.H} we have $H_S^{f}(1)=0$.  

$n>1$. 
Let $1 \leq k < |S|$. By Item (i) in the definition of the criterion function, 
if $\pi \in \Pi_{k+1}$ and we change one of the optimality deficiencies 
of the elements to 0 then the criterion function $f$ does not increase. 
Hence the minimum of $f$ for a partition in $\Pi_k$ is not larger than 
the minimum of $f$ for a partition in $\Pi_{k+1}$. 
Therefore $H_S^{f}$ is monotonic non-increasing.
For $k=n$ the multiset $S$ is partitioned into singleton sets which all 
have optimality deficiency 0. Hence $H^f_S(n)=0$.
\end{proof}

\begin{proof} of Lemma~\ref{lemma.lb}
Choose $x\in \{0,1\}^m$ and $S$ with $|S|=n$ such that
$S=\{y: |y|=m$ and $y$ equals $x$ with the $i$th bit flipped 
$(1\leq i \leq n) \}$. Then for each $y \in S$ we have
$K(S) = K(y)+O(\log n)$.
Therefore $\delta(S,y) = K(S) + \log n-K(y)=O(\log n)$ for all $y \in S$.
Hence $H_S^{f}(1)=O(\log n)=O(\log K(S))$. For every $k$
($1 < k \leq n$) we describe the partition $\pi \in \Pi_k$
which witnesses  $H_S^{f}(k)$ by giving $S$ in $K(S)$ bits,
the integer $k$ in $O(\log n)$ bits and an $O(1)$ program.
This program does the following: given $k$ and $S$ it generates
all finitely many partitions $\pi \in \Pi_k$. 
A partition $\pi \in \Pi_k$ of $S$ divides it into, say, $S_1, \ldots, S_k$.
By the symmetry of information law \cite{Ga74} we have
$K(S)=K(S_i)+K(S|S_i) +O(\log K(S))$ or
$K(S_i) \leq K(S)-O(\log K(S))$. For every $y \in S_i$ therefore
$\delta(S_i,y) = K(S_i) +\log |S_i| -K(y) \leq
K(S) - O(\log K(S))+\log |S| -K(y)= \delta(S,y) -O(\log K(S))$. 
Since $1 \leq k \leq n$ this proves the lemma.
\end{proof}

\begin{proof} of Lemma~\ref{lemma.ub}
Let $S=\{x_1, \ldots ,x_n\}$ with $K(x_i)= im/n$ for $1 \leq i \leq n$.
(This is possible since all $n$ members of are strings of 
length $m$ and they can have 
complexity varying continuously between at least $m$ and close to $0$.)
Since for each finite multiset $A$ and $x \in A$ we have 
$\delta(A,x) = K(A)+\log |A| - K(x)$ and therefore 
\[
\max_{x \in A} \delta(A,x) - \min_{x \in A} \delta(A,x) =
\max_{x \in A} \{K(x)\} - \min_{x \in A} \{K(x)\}. 
\]
For a $k$-partition of $S$ at least one $S_i$ in the partition has 
cardinality at least $n/k$. Therefore, if $n/k > 1$ then by the displayed
equality $H_S^{f}(k) > m/n$. This holds for $k=1, \ldots , n-1$.
For $k=n$ all parts $S_i$ in the partition are singleton sets and hence
$H_S^{f}(k) = 0$. 
\end{proof}

\begin{proof} of Lemma~\ref{lemma.4}.
Similar to the proof of Lemma~\ref{lemma.ub}.
\end{proof}
\begin{proof} of Lemma~\ref{lemma.lbsd}. 
Similar to proof in Lemma~\ref{lemma.lb}.
\end{proof}
\begin{proof} of Lemma~\ref{lemma.ubsd}.
Similar to the proof of Lemma~\ref{lemma.ub}.
\end{proof}

\subsection{Computing the number of clusters}\label{sect.number}

To determine the number $K$ of clusters in data $S$ we compare
a cluster structure function used on $S$ with the same cluster function
on {\em reference set} of  $|S|$ data distributed uniformly.
We do this comparison as the logarithm of the ratio. Using the 
cluster function $H^f_S$  on the 
data set $S$ the number $K$
of clusters in $S$ is the $k$ where the log-ratio $D^f(k)$ is
greatest. Formally
\begin{align*}
D^f (k) & = \log H^f_N(k)-\log H^f_S(k) \\
K & = \arg\max_k D^f (k), 
\end{align*}
with the reference placement 
is the uniform distribution of $|S|$ data samples
over the range spanned by $S$.
For example if $S$ is a set of numbers
than its range is the interval $I=[\min S, \max S]$. Note that every set
$S$ is represented in a computer memory as a finite set of finite 
strings of 0's and 1's and that therefore $\min S$ and $\max S$ are
well defined. 
Divide I in $n$
equal parts $I_1, \ldots ,I_n$ with $\bigcup_{i=1}^n I_i = I$ and
$I_i \bigcap I_j = \emptyset$ for $1 \leq i \neq j \leq n$.
Item $i \in N$ is positioned in the middle of subinterval 
$I_i$ ($1 \leq i \leq n$).

To deal with the incomputability of the function $K$
we approximate
$K$ from above by a good compressor $Z$. 
If $x$ is a string then
$Z(x)$ is the length of the by $Z$ compressed version of $x$.
The  function $Z$ is by construction a computable function, 
even a feasibly computable one
(for example $Z$ is bzip2 or some other compressor).
Because $K$ is incomputable there are strings $x$
such that $K(x) \ll Z(x)$ and the difference $Z(x)-K(x)$ is incomputable.
However for natural data we assume that they encode no universal computer
or problematic mathematical constants like the ratio of the
circumference of a circle to its diameter $3.14 \ldots .$ We assume
that for the natural data we encounter the compression by $Z$ has
a length which is close to its prefix Kolmogorov complexity.
The same holds for a multiset $A$ of strings.
We represent $A = \{x_1, \ldots , x_n\}$ as a string 
$s(A)=1^{|x_1|}0x_1 \ldots 1^{|x_n|}0x_n$ with 
$|s(A)|=|x_1 \ldots x_n|+O(\log |x_1|+ \cdots + \log | x_n|)$.

For a partition $\pi \in \Pi_k$ of $S$ ($|S|=n$) we compute
the $\delta(S_i,x)$'s by computing
$Z(S_i)$ ($1 \leq i \leq k$) and $Z(x)$ for all $x \in S$.
To do so we require at most $k+n$ compressions. We write ``at most''
since a member of a multiset $S$ can occur more than once.

%

\section{Probabilities Among Members of Clusters}\label{sect.meaning}
By Lemma~\ref{lem.sufficiency} 
a part $A$, with more than two members, of a witness partition of $S$ 
can not be a sufficient statistic for all of its elements.
In clusters the
members of a cluster typically share some characteristics but not 
all characteristics. It turns out that in an appropriate sense
the members of a cluster are nonetheless 
probabilistically close. 

We define a conditional probability
of $n$-bit strings following \cite{Mi15}. We start with the 
unconditional probability.
Let a finite set $A$ of $n$-bit strings be chosen 
randomly with probability
${\bf m}(A)= 2^{-K(A)}$, and subsequently $x \in A$ is chosen 
with uniform probability 
from $A$, that is, $x$ is chosen with probability ${\bf m}(A)/|A|$.
(Since $K(x)$ is a length of a prefix code we have by Kraft's inequality
\cite{CT91} that $\sum_x 2^{-K(x)} \leq 1$. Hence ${\bf m}$ is a
semiprobability. A semiprobability is just like a probability
but may sum to less than 1. The particular semiprobability 
${\bf m}$ is called {\em universal}
since it is the largest lower semicomputable  
semiprobability \cite{Le74}. In absence of any information about $A$
we can assign ${\bf m}(A)$ as its probability. 
Properties are discussed in the text \cite{LV08}). 
\begin{definition}\label{def.p}
\rm
For each $y \in A$ we define the {\em conditional probability} $p(y|x)$ by
\[
p(y|x) = \frac{ \sum_{A \ni x,y} {\bf m}(A)/|A|}{\sum_{A \ni x} {\bf m}(A)/|A|}
\] 
\end{definition}
We show below that all pairs of strings in a 
part of a witness partition of multset $S$ of $n$ strings  
have an expectation of the conditional $p$-probability with respect 
to each other
which is at least $2^{-H^{f}_S}(k)$ for some $k \leq n$.
Hence the smaller $H^{f}_S (k)$ is the more all strings in a part of a
witness partition of $H^{f}_S(k)$ have a large conditional probability
with respect to each other: they form a cluster.

\begin{theorem}\label{theo.cp}
Let $S \subseteq \{0,1\}^n$ (consider only $n$-length strings) and
a witness $k$-partition of $S$ for $H_S^{f}(k)$ that divides $S$ into parts 
$S_1, \ldots , S_k$. The expectation taken over a random variable
$p(y|x)$ for pairs
$x,y \in S_i$ for some $i$ ($1 \leq i \leq k$) is
${\bf E}[p(y|x)] \geq 2^{-H_S^f(k) - O(\log n)}$ and
${\bf E}[p(y|x)]$ becomes at least $(1/n)^{O(1)}$ for $k \rightarrow n$.
\end{theorem}
\begin{proof}
The parts of a witness to $H_S^f(k)$
form clusters because intuitively if the conditional probabilities 
in Definition~\ref{def.p}  
of different strings in a part of the witness partition are small
then the conditional Kolmogorov complexities are small:
\begin{claim}
\[
p(x|y)= \frac{\Theta({\bf m}(x,y))}{\Theta({\bf m}(x))}= 2^{-K(x|y)-O(\log n)}.
\]
\end{claim}
\begin{proof}
Start from Definition~\ref{def.p}.
The first equality holds by the following reasoning: 
since $\sum_{A \ni x} {\bf m}(A)/|A| 
= \Theta ({\bf m}(x))$ because
the lefthand side of the equation is a lower semicomputable function of $x$
and hence  it is $O({\bf m}(x))$; moreover if $A=\{x\}$ then the lefthand side
equals ${\bf m}(x)$. The same argument can be used for the pair $\{x,y\}$.
The second equality uses the coding theorem \cite{Le74} which states 
${\bf m}(x) = 2^{K(x)+O(1)}$ and the symmetry of information law \cite{Ga74}
which shows both the trivial $K(x,y) \leq  K(x)+K(y|x)$  and
$K(x,y) \geq  K(x)+K(y|x)-O(\log K(x,y))=K(x)+K(y|x)-O(\log n)$. 
The $\Theta$
order of magnitude is an $O(1)$ term in the exponent and absorbed in the 
$O(\log n)$ term.
\end{proof}

The conditional probabilities of pairs of strings in a part
of a $k$-partition of $S$ which is a witness to $H^{f}_S(k)$ satisfy
the following. 
By \cite[Theorem 5]{Mi15} if $x,y \in S_i$ for a {\em particular} $i$ 
($1 \le i \leq k$) 
and $\delta(S_i,x) \leq d$ then $p(y|x) \geq 2^{-d - O( \log n)}$,
while if $p(y|x) \geq 2^{-d}$ then $\delta(S_i,x)\leq d+O(\log n)$.
Hence $p(y|x)= 2^{-\delta(S_i,x) \pm O(\log n)}$. 
The expectation of $p(y|x)$ over $S_i$ is given by 
\begin{align*}
{\bf E}[p(y|x)] & = 1/|S_i| \sum_{x \in S_i}2^{-\delta(S_i,x) \pm O(\log n)}
\\& \geq 2^{-\sum_{x \in S_i} (\delta(S_i,x) \pm O(\log n))/|S_i|}
\\&= 2^{- \mu_{S_i} \pm  O(\log n)},
\end{align*}
using in the second line the inequality of arithmetic and geometric means.
This implies that if $x,y \in S_i$ for {\em some} $i$ ($1 \leq i \leq k$) 
then the expectation of $p(y|x)$ over all $S_i$ in a witness 
partition of $S$ is given by
\begin{align*}
{\bf E}[p(y|x)] & =(1/k) \sum_{i=1}^k 2^{- \mu_{S_i} \pm O(\log n)} 
\\& \geq 2^{-(1/k)\sum_{i=1}^k (\mu_{S_i}\pm O(\log n))}
\\& \geq 2^{- H_S(k)\pm O(\log n)},
\end{align*}
using in the second line again the inequality of arithmetic 
and geometric means and in the third line that $H_S(k) \geq 
(1/k)\sum_{i=1}^k \mu_{S_i}$ by Definition~\ref{def.H}.
Since $H_S(n)=0$ we have 
${\bf E}[p(y|x)]$ is at least $(1/n)^{O(1)}$ for $k \rightarrow n$.
\end{proof}

\begin{remark}
\rm
Roughly, the smaller $H_S$ becomes the larger the conditional
probabilities of the elements in a part of the witness partition become.
\end{remark}

\section{Related Literature}\label{sect.rlit}

This paper extends previous work in the field of algorithmic statistics 
\cite{GTV01,VV04,Vi06}. The applications build particularly on the field of semi-supervised spectral learning \cite{Kamvar,NJW,cohenNM}. Most previous approaches to estimating the number of clusters in a dataset utilize probabilistic statistical modeling of the data. The Bayesian and Akaike information criteria  both formulate the question w.r.t. underlying distributions estimated either parametrically or empirically \cite{TK08,bayes_bic,bayes_point,bayes_sparse}.  Bayesian methods are well suited when the likelihood function and prior probabilities are known. In comparison, the algorithmic statistics approach proposed here works with the particular dataset rather than probabilities across a hypothesized population. \color{black} Recently, alternative approaches based on characteristics of the specific data set in question, rather than a population-level model, have been considered \cite{bayes_simplical,bayes_xval}. These include particularly the widely used Gap statistic \cite{Tib01} that is very similar in spirit to the implementation described here. The connection between the gap statistic and the field of algorithmic statistics was one of the key motivators for this work \cite{ARC1}. The gap statistic compares the spatial characteristics of the data being clustered to that of a randomly generated reference distribution. Our approach is similar in both theory and practice to the gap statistic. Many of the advantages of the two approaches are shared. Both are effective when K=1, that is there are no meaningful clusters among the data. Both are reasonably efficient to compute.  DBScan combines the clustering and K estimation into a single task, and provides parameters for fine tune control \cite{dbscan}. In theory the cluster structure function might be used in an automated parameter search with such an algorithm.  

In the computational biological microscopy image analysis area we build on previous work for optimally partitioning connected components of foreground pixels into elliptical regions \cite{Winter}. A key advantage of the cluster structure function compared to all other approaches is the very broad and powerful theoretical structure of Kolmogorov complexity and Algorithmic Statistics. The techniques are generally parameter free, beyond the selection of a suitable compression algorithm. In theory it will be possible to automatically identify the optimal compression by considering ensembles of algorithms and choosing the best results among them via the structure function.

\section{Example Applications}\label{sect.eappl}

\subsection{How many different digits are in a MNIST digit set?}

Here we apply the optimality deficiency to estimating the number of different digits in a set of digits sampled from the MNIST 
handwritten digits dataset. Classification of the MNIST digits using supervised
learning techniques is well studied but there has been little application of unsupervised learning to this
problem. One key challenge is establishing a ground truth number of different classes. Different
styles of handwriting were taught at different times in different locations. These differences are likely
reflected in the underlying data as distinct categories, even within digits of the same class label. Another
challenge is the difficulty in unsupervised classification of digits even when the correct number of classes 
is known. While supervised solutions for the MNIST digit classification are extremely accurate, unsupervised clustering of MNIST digits is still a difficult problem. 
The MNIST data has been normalized 
to 28x28 8-bit grayscale (0,..,255) images.
The MNIST database contains a total of 70,000  handwritten digits consisting of 60,000 training examples and 10,000 test examples.  Originally the input looks as Figure~\ref{fig.digits}.

\begin{figure}[b]
  \begin{center}
  \includegraphics[width=2.5in]{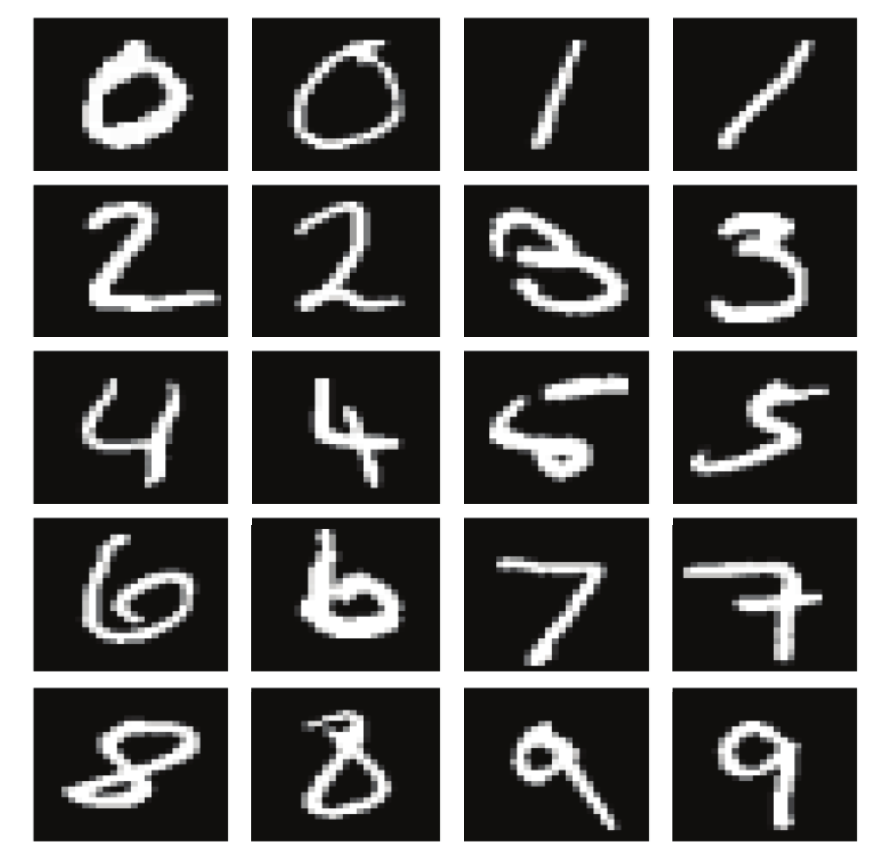}
  \end{center}
  \caption{Example MNIST handwritten digits}
  \label{fig.digits}
\end{figure}

We apply the cluster structure function to the question of estimating how many 
different MNIST digits are represented in a large set. 
Configure a sampler to choose a set of 100 digits at a time randomly given a fixed $K$ value. In each digit set, the cardinality of each digit
is given by $\lceil \frac{100}{K} \rceil$. For $K=3$ there would be 33 of each 
digit [0,1,2]. For $K=10$, there would be 10 of each digit $[0,9]$. Spectral clustering
\cite{NJW} is used to cluster MNIST digit sets \cite{NCDM}. 
The spectral clustering approach starts 
with the matrix of pairwise normalized compression distances (NCD) as in \cite{CV04}  
among all pairs of digits. We used the free lossless image 
format (FLIF) compressor \cite{FLIF} for the MNIST digits, and found it to significantly outperform the previously used BZIP and JPEG2000 compressors. After computing the NCD matrix between digit pairs, the classic spectral clustering algorithm \cite{NJW} is applied. Table \ref{clusterResult} shows the results of spectral clustering for all ten digits. Clustering accuracy for all ten digit types [0..9] was 46\% with 95\% confidence intervals of [.4622,.4657] obtained via bootstrapping \cite{TK08}.

\begin{table*}[h]
  \centering
  \includegraphics[]{"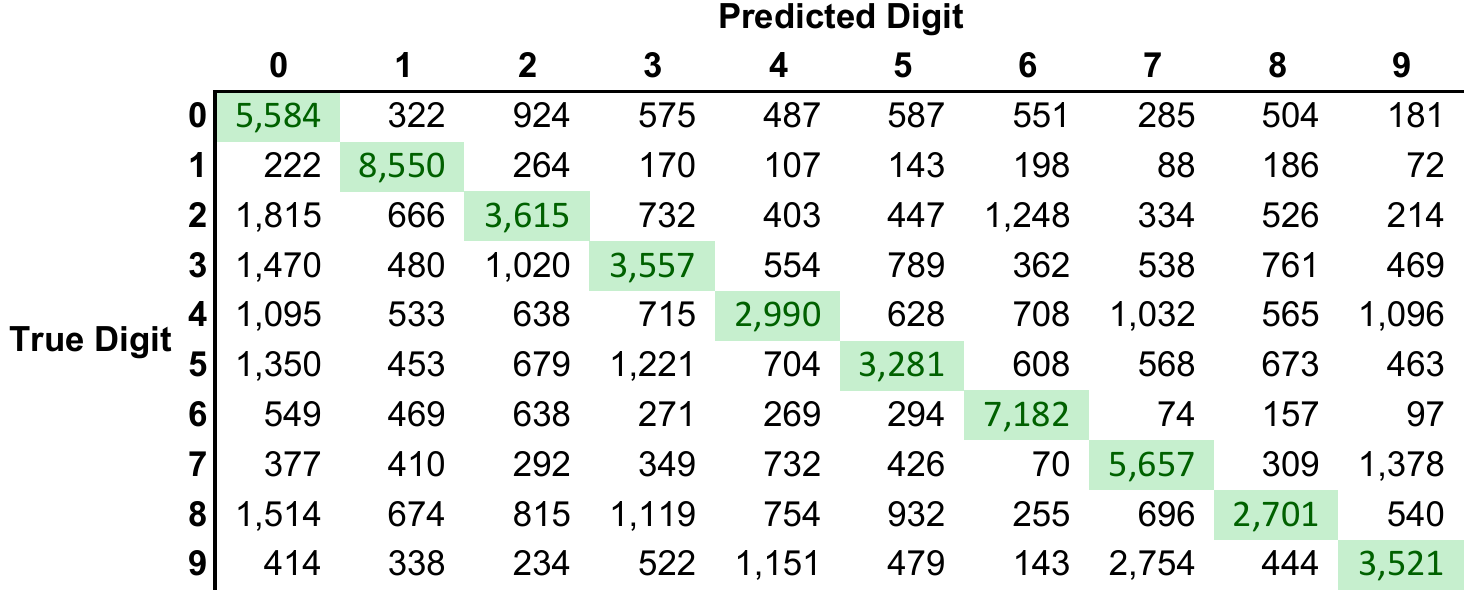"}
  \caption{\label{clusterResult} Confusion matrix for spectral clustering sets of random digits. Each digit set contains 5 of each digit [0,9]. The digit sets are clustered into 10 clusters. For evaluation, each cluster is labeled with the mode (most common element) of the true digit values in that cluster. This was repeated ten thousand times. Overall accuracy of clustering the ten digit classes is 46\%.
  } 
\end{table*}

To compute the CSF function following Section~\ref{sect.number} we proceed as follows.
For each digit set $S$, we generate Cluster Structure Function (CSF) curves. 
The digit set $S$ is clustered at different values of $K$. 
Random samples of the data forming subset  
$\tilde{S} \subseteq S$ are chosen iteratively. 
Statistics of the pointwise CSF curves are formed from the random samples 
$\tilde{S}$. The results here were generated using 1000 random samples of 
each digit set as follows. For each digit set $S$ compute the pairwise NCD matrix $D$ between all elements of $S$ using the FLIF image compression. For each $K$  on $[1, \ldots, K_{\max}]$ use spectral clustering to partition the elements of $S$ into $K$ groups. Each cluster (partition) is labeled $A_p=\{x_1,x_2,...,x_{|A_p|}\}$ and $|A_p|$ is the number of points in cluster $A_p$. After the points have been clustered for a particular value of $K$, pick subsets at random from $S$ 
to form $\tilde{S}$, $|\tilde{S}|=5*K$. Using the cluster assignments for each of the randomly selected points, compute the optimality deficiency for each random sample across each of the $K$ clusters $A_p$

\begin{equation}\label{eq.deltaCluster}
  \delta(A_p,x_i)=
  \begin{cases} 
    0 & |A_p|<2 \\
    Z(A_p)-Z(x_i)+\log |A_p| & |A_p| \geq 2,
  \end{cases}
\end{equation}
where $Z(A_p)$ is the size in bytes of the FLIF compressed image formed by concatenating all of the digit images in $\tilde{S}$ belonging to cluster $A_p$ and $Z(x_i)$ is the size in bytes of the compressed image corresponding to digit $x_i$. We write 
$\delta(A_p)=\{\delta(A_p,x_1),\delta(A_p,x_2),...\delta(A_p,x_{|A_p|})\}$ to denote the set of optimality deficiencies for each $x_i \in A_p$.
After computing $\delta(A_p)$ for each cluster from the digit set 
subsample $\tilde{S}$, the results of \ref{eq.deltaCluster} are combined 
to compute the related cluster structure function:

\begin{equation}\label{eq.HS}
  H_{\tilde{S}}(K) = \frac{\sum\limits_{p=1}^{K_{\max}}\log_2(\max(\delta(A_p)) - \min(\delta(A_p)) + 1)}{K_{\max}}.
\end{equation}

The final cluster structure function (CSF) curve is then generated using 
the mean and standard deviation of $H_{\tilde{S}}(K)$ across all random 
subsamples $\tilde{S}$ and $K$ values. The optimal value of $K$ for such 
a CSF curve is chosen using the technique proposed in \cite{Tib01}, as the 
first value of $K$ where the CSF curve decreases more than one standard 
deviation from the previous value.  A robust estimator for standard deviation may be useful in identifying the minimum value of the cluster structure function for some applications. \color{black}
Figure \ref{fig:odCurves} shows two example CSF curves. In the left panel, 
the correct  value is obtained at $K^{pred}=K^{true}$. In the right pane 
of Figure \ref{fig:odCurves}, the selected value is obtained 
at $K^{pred}=5$ and does not match the $K^{true}=9$ correct value, 
although there is a minor decrease at nine for that example. 

\begin{figure*}[h]
  \includegraphics[width=1\textwidth]{"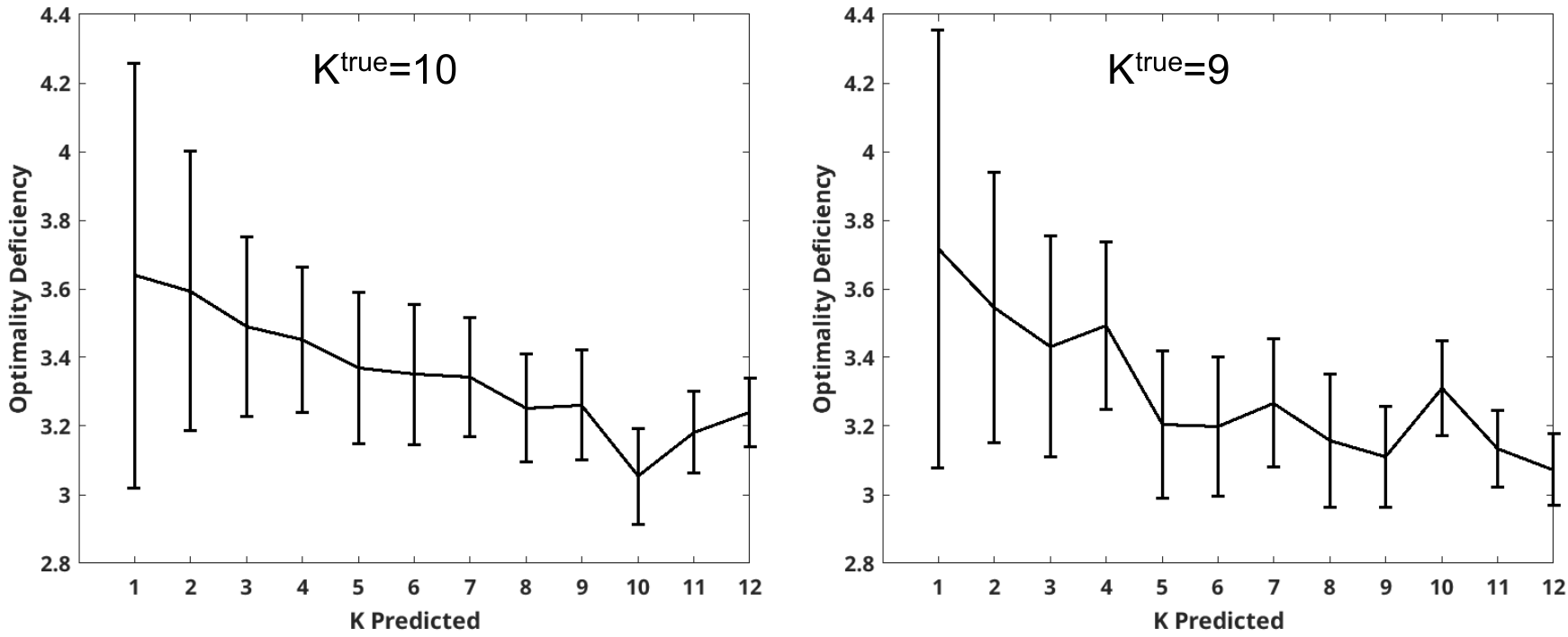"}
  \caption{\label{fig:odCurves} Curves showing mean and standard deviation of the cluster structure function (CSF) for two different digit sets. Subsets are chosen repeatedly from each digit set, and clustered into $K$ groups. The value of $K$ is chosen as the first $K$ that is one standard deviation smaller than the previous value. The left curve selects the value $K=K^{true}$, correctly identifying the value of $K$ corresponding to the number of different digits in the set. The right curve incorrectly selects $K=5$. 
  } 
\end{figure*}

The minimum of the empirical CSF curve is not in itself significant 
since the ideal theoretic CSF curve is monotonic non-decreasing and the minimum 
is always at 0 (Lemma~\ref{lemma.1}).
What makes the minimum possibly a little meaningful is when the minimum 
occurs just after a 
sharp decrease in the CSF curve. The empirical CSF curves of 
Figure~\ref{fig:odCurves} seem in contradiction 
with Lemma~\ref{lemma.1}. The curves in the figure  
roughly follow Lemma~\ref{lemma.1}, but they are the results of 
several heuristics so they may not be perfectly monotonic non-decreasing. 
The heuristics are among others:
approximation from above of the non-computable Kolmogorov complexity,
the spectral heuristic of finding the number of clusters rather than 
inspecting all the subsets of the data, and repeated random sampling of a 
subset $\tilde{S} \subseteq S$ computing the CSF curve of each $\tilde{S}$ 
and taking the average. To identify the number of clusters in the 
data one takes the number following the sharp decrease 
of the  CSF curve. Here the criterion to select the clusters is optimally 
satisfied.

Table \ref{tbl.nistUnsupervisedK} shows the average and standard deviations from subsampling digit sets of varying $K^{true}$. Digits sets with $K^{true}=1$ have a higher $K^{pred}$ value, and a much higher standard deviation compared to digits sets with other values of $K^{true}$. Omitting digit sets with $K^{true}=1$ significantly increases the correlation between the selected point on the CSF curve and $K^{true}$. For the CSF, the correlation $r$ between $K^{true}$ and $K^{pred}$ for $K^{true}>1$ is $r=0.93$, with a p-value $p=3e-4$. For the Gap Statistic, $r=-0.84$ ($p=5e-3)$. 
Based on that observation, a shallow feedforward neural network was used to map the CSF curves to a predicted value $K^{pred}$.

\begin{table*}[t]
  \centering
  \includegraphics[]{"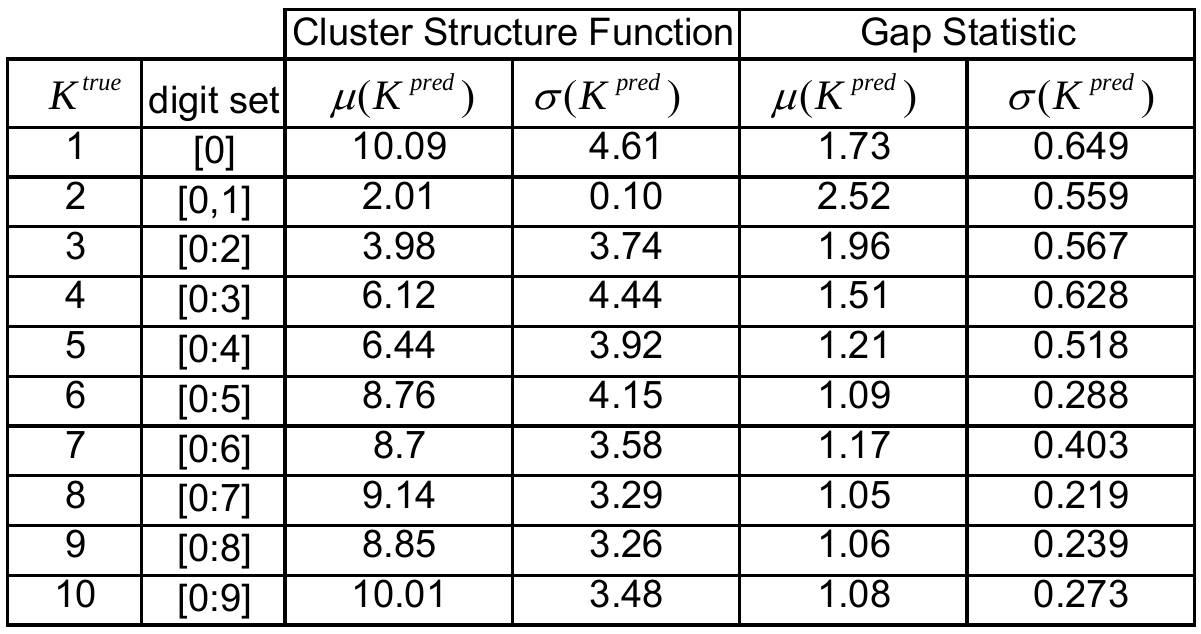"}
  \caption{\label{tbl.nistUnsupervisedK} Unsupervised cluster structure function (CSF) (left) and Gap Statistic (right) estimates of the number of unique digits $K$ in a MNIST digit set. Both CSF and Gap Statistic predictions $K^{pred}$ are correlated with $K^{true}$ except in case $K=1$ (where both exhibit much higher standard deviation). Omitting $K^{true}=1$, the CSF correlation is $0.93$ ($p=3e-4$) and the Gap Statistic correlation is $-0.84$ ($p=5e-3)$. 
  }   
\end{table*}

The approach is now to use the 20 element vector composed of the mean and standard deviations of the CSF curves evaluated at the numbers of clusters 
$K=[1..10]$ as a feature vector to identify the optimal value of $K$. We use one thousand examples each of digit sets from $K=[1..10]$ as training data (ten thousand total digit sets). Using the MATLAB patternet() classifier with all default parameters, a shallow feed forward neural network with 20 input layer nodes, 10 hidden layer nodes and 10 output layer nodes is trained using ten thousand digit sets, one thousand examples each from $K \in [1..10]$.  We classify 100 unknown digit sets. When the classification confidence is low, we repeat the sampling, selecting a new $S$ up to 10 times and average the results to form the prediction. Table \ref{tbl.nistResultsCM} shows the resulting predictions. The vertical axis of the table represents $K^{true}$, the horizontal axis represents $K^{pred}$. Elements on the diagonal represent correct classifications. Overall accuracy, measured as the percentage of non-zero results that fall on the diagonal of the confusion matrix is 86\% with a 95\% confidence interval [0.84,0.88] established by bootstrapping. We used the same procedure on the mean and standard deviation values obtained from the Gap Statistic (as in Table \ref{tbl.nistUnsupervisedK}) and obtained an accuracy of 54\% [0.51,0.57].

\begin{table*}[]
  \centering
  \includegraphics[]{"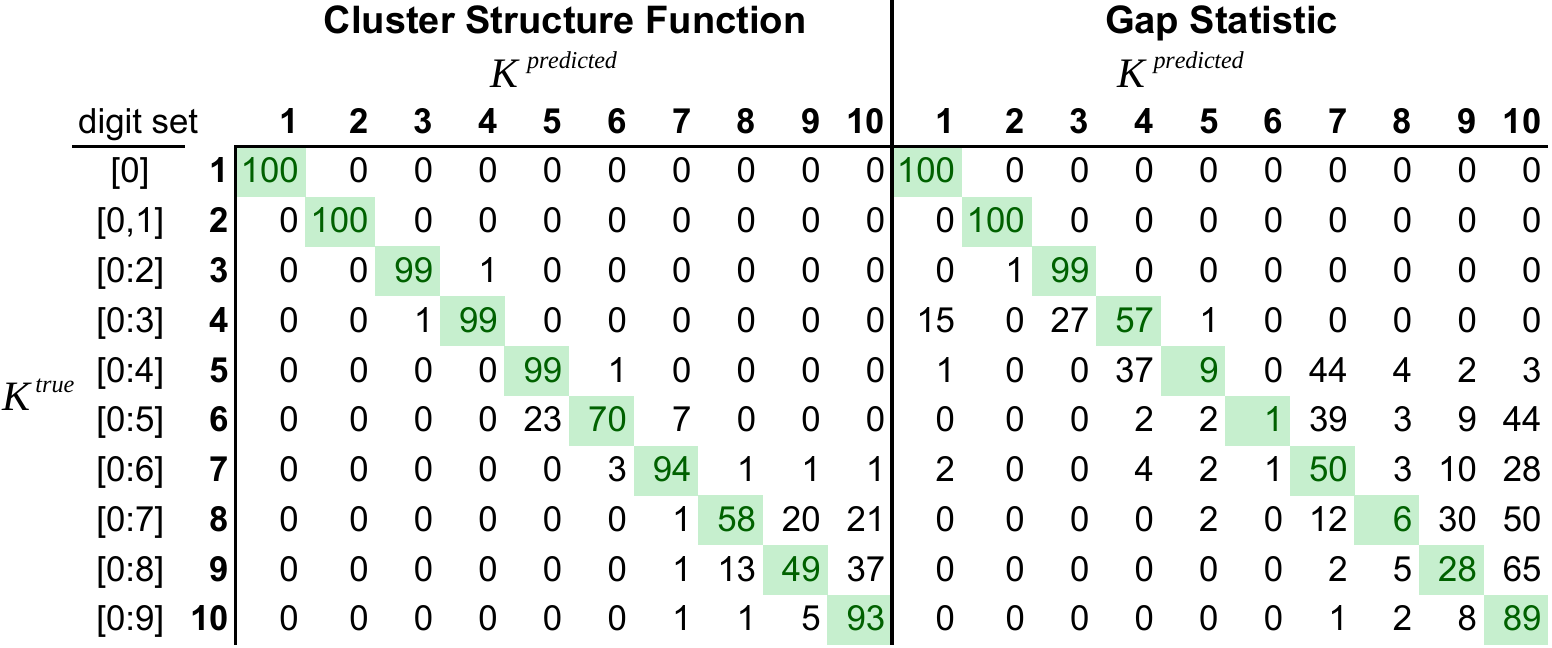"}
  \caption{\label{tbl.nistResultsCM} Supervised cluster structure function (CSF) (left) and Gap Statistic (right) estimates of the number of unique digits $K$ in a NIST digit set. Each digit set contains 100 digits, split equally among the $K$ digit classes. The algorithm is given a digit set sampler that can pull repeatedly from the same distribution ($K$ value) with the goal of estimating $K$. The results here were generated by classifying one hundred each of digit sets with $K^{true} \in [1..10]$. A 20-element vector consisting of mean and standard deviations of the CSF and the Gap Statistic was the input to a shallow feed-forward neural network. Overall accuracy for the CSF was 86\% [0.84,0.88] and 54\% for the Gap Statistic [0.51,0.57].
  }   
\end{table*}

\subsection{Cell Segmentation}
   
Cell segmentation is the identification of individual cells in microscopy images. 
The identification of cell nuclei in microscopy images is 
an important question.  Human stem cells (HSCs) are particularly challenging to segment as the cells are highly
adherent, forming in naturally densely packed colonies. HSC  colonies, or groups of touching cells, 
consist of dividing and differentiating cells that present a wide variety of sizes and shapes. 
The large morphological  variation arises from both the presence of cells in developmental 
states and the mechanical interaction among adjacent cells deforming their shape, texture, and 
behavior \cite{SCR15,TMI18}. Timelapse microscopy of living cells further complicates the problem, 
requiring reduced imaging energy to lessen phototoxicity, and also introducing temporal variations 
due to imaging as well as cell and colony
appearance variability. It is much easier to segment cells that all have a similar appearance,  for example shape and size. Here we present a technique for combining multiple simultaneous segmentations of the same image, each with varying 
underlying segmentation parameters. We refer to the collective set of segmentation results
as an ensemble. The segmentations in the ensemble are combined by using optimality deficiency to select among
overlapping segmentations. We use a previously described unsupervised underlying segmentation \cite{SCR15,BI16,TMI18} 
that takes a single parameter of cell size in $\mu m$. The method works as follows. The segmentations are run across a range of 
expected radius values. The results are combined, with cells that overlap each other placed in common "buckets".
The question is then to choose the optimal number of cells $K$ in each bucket. Every segmentation is given
a score based on its appearance and how well it captures the underlying pixels. Here we apply the approach to the question
of identifying elliptical cells or nuclei. Rather than using compression-based similarity, the score is built on an 
appearance model. 

The segmentation model expects cells that are convex, brighter in the interior 
compared to the exterior, and to contain a well defined boundary between a bright interior and dark exterior. 
Given a particular cell segmentation $C$, the score is a combined measure of convex efficiency, background
efficiency and boundary efficiency. The term efficiency describes a normalized measure capturing how close
to the model the data achieves. The convex efficiency is defined as 
\[
e_{convex}(C)=\frac{|C|}{|C_{convex}|},
\]
where $|C|$ is the area (volume) of segmentation $C$ and $|C_{convex}|$ is the area of the convex hull of $C$.
The boundary efficiency is computed from the normalized ([0,1]) image pixel values, defined as
\[
e_{boundary}(C)=1-mean(R(\beta(C))-T(\beta(C))),
\]
where $R(\beta(C))$ is the maximal intensity in the region surrounding the boundary voxels $\beta(C)$, and $T(\beta(C))$ 
is the mean adaptive threshold value for voxels along the boundary. The background efficiency is defined as
\[
e_{background}(C)=\frac{mean( I(C)-T(C) )}{mean(I(\hat{C})-T(\hat{C}))},
\]
where $I(C)$ is the source image, $T(C)$ is the adaptive threshold image of segmentation $C$, and $\hat{C}$ represents the 
image background. The final segmentation score is the sum of the three scores, 
\begin{equation} \label{eq.cellScore}
e_C=e_{convex}(C)+e_{boundary}(C)+e_{background}(C).
\end{equation}
After each cell has been scored, the goal is to select the set of non-overlapping segmentations from the ensemble that maximize
the sum of the individual segmentation scores. This is equivalent to selecting  
the $H_S(k)$ from Equation \ref{eq.H}  
where the $\delta(A,x)$ in Equation \ref{eq.delta} are approximated by the individual cell segmentation scores. Figure \ref{fig:ensembleSegmentation} demonstrates
the ensemble segmentation for a colony of HSCs imaged using a fluorescent nuclear marker (H2B). 

Quantitative validation for the ensemble segmentation approach was done using ground truth 
data from the cell tracking challenge \cite{CTC} reference datasets. Twelve time-lapse datasets in 2-D and 3-D of
live cells were processed using the ensemble segmentation with an empirically selected range of radius parameters. 
Ground truth scores were obtained for each radius parameter setting run separately and also for the ensemble segmentation.
We consider the detection (DET) score here, as our concern is not primarily the accuracy of pixel assigned to each 
segmentation, but rather that we detect the correct number of cells in each frame. We use the training
movies for validation because our method is unsupervised and training is not required. Our results are competitive
on these movies with the supervised algorithms evaluated on the testing challenge datasets.
In each of the 12 movies, the ensemble segmentation outperformed the best result selected from 
segmentations run separately. The results for the optimality deficiency based ensemble segmentation 
were statistically significantly better compared to the best score obtained from the single radius 
segmentation data for both the detection (DET) ($p=5e-4$, Wilcoxon paired sign-rank test) 
and tracking (TRA) scores ($p=2e-3$). This is significant because the best radius result varied even within
pairs of movies from the same application type, showing the value of the ensemble segmentation approach. 
Table \ref{resultsSegmentation} shows the results for the ensemble classification as well as the best and worst performing individual segmentation for each of the datasets processed here.

\begin{figure*}[]
\includegraphics[width=1\textwidth]{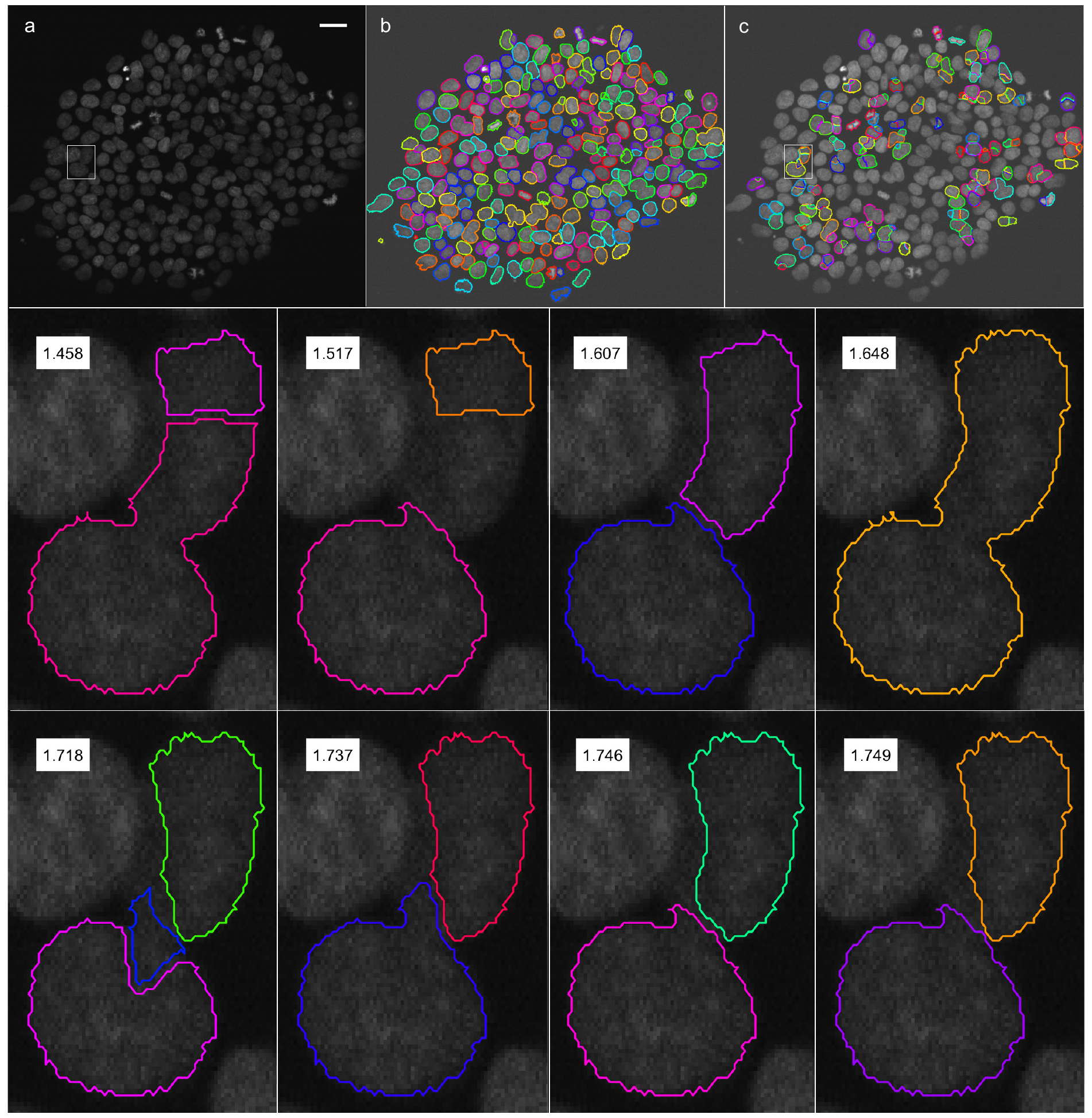}
\caption{\label{fig:ensembleSegmentation} The ensemble segmentation combines results from different segmentation algorithms using
the optimality deficiency to select the best results for overlapping segmentations. Frame segmentations are run at each of a 
range of different parameter values. The resulting segmentations
are each treated as a possible clustering of the underlying pixels into objects. An example is shown here for a single 
image frame taken from a 1200 frame movie showing the development of live human stem cells (HSCs). The top row shows 
a raw image (a), the final segmentation results (b) and the overlapping ensemble regions (c). The bottom two rows show 
different possible combinations of segmentation results from the region shown in the rectangle in (a) and (c). The segmentation results
are scored from worst (lowest score) to best (highest score). The optimal set of segmentation results are selected using a greedy optimization to 
maximize the scores in each overlapping region. Segmentation scores are generated from the convexity, boundary and
background efficiencies.  } 
\end{figure*} 
 
\begin{table*}[]
  \centering
  \includegraphics[]{"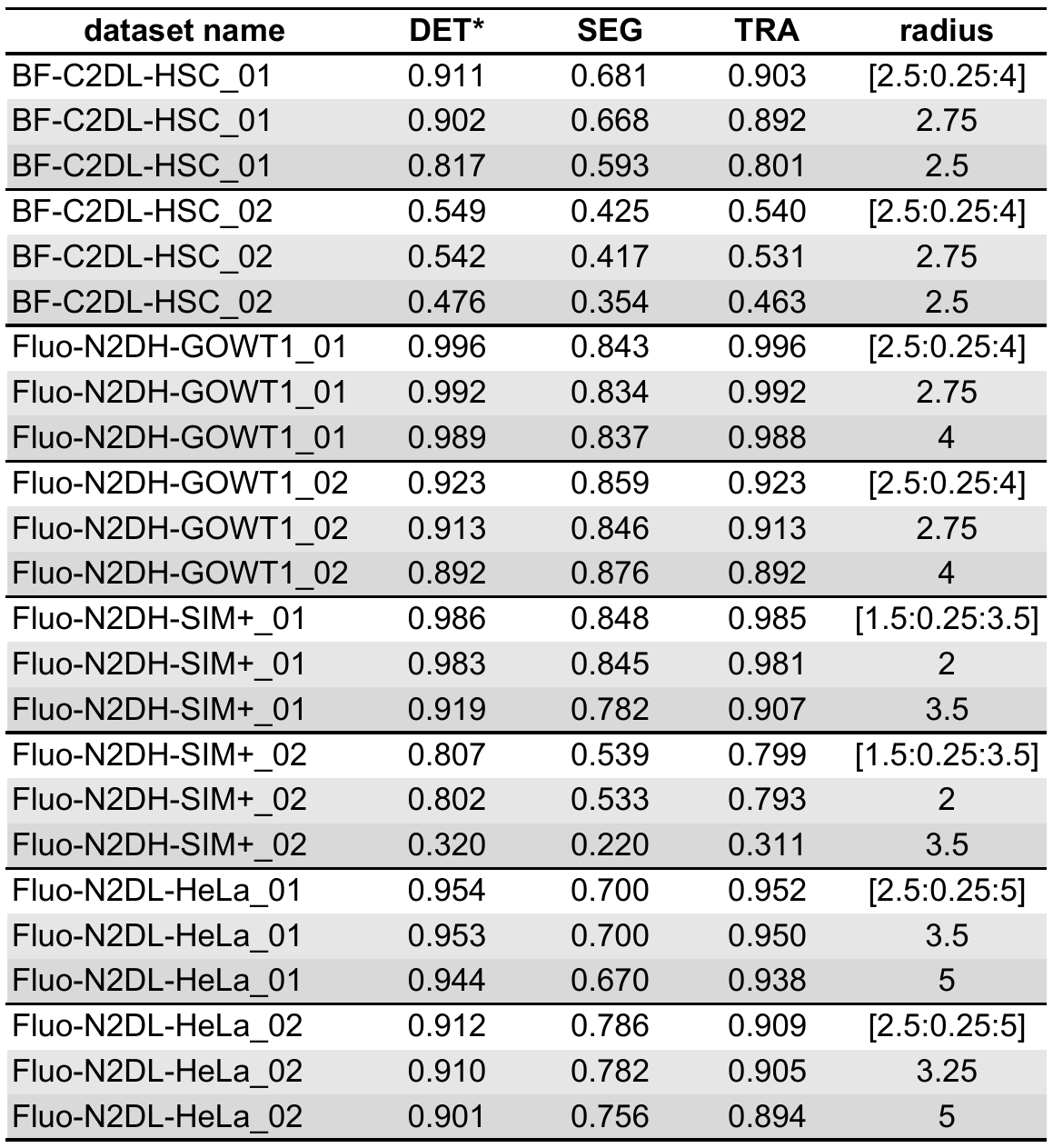"}
  \caption{\label{resultsSegmentation} Ensemble segmentation combines results from segmentation algorithms run at different parameter settings on 2-D and 3-D image data. Optimality deficiency estimates  the number of cells $K$ in each region of overlapping segmentations. The approach here is optimizing the detection (DET) metric for the cell tracking challenge datasets. The first row in each group shows the ensemble results and radius parameter settings, the subsequent two rows show the best and worst performing single segmentations. The ensemble segmentation significantly outperforms the best individual segmentations ($p=5e-4$). 
  } 
\end{table*}

\subsection{Synthetic Dataset}
We evaluate the performance of the cluster structure function using synthetic data generated as random points from $K=3$ different 2-D standard normal distributions, each with covariance $\Sigma=[1,0;0,1]$. Position the $K=3$ clusters along the x-axis at $x=[0,r,2*r]$ with cluster spacing $r=[0.5:0.25:1.5]$. In each of the 100 trials, generate 1e4 points from each of the $K=3$ distributions. Supplementary Figure 1 shows a histogram of an example synthetic dataset with cluster spacing $= 1.0$. To evaluate the cluster structure function, approximate $K(A)-K(x)$, as in eqn. \ref{eq.delta} using the Euclidean distance between point $x$ and the centroid of cluster $A$. As in the examples above, we include only the points that fall within one standard deviation of the centroid for each cluster and then average this result across each cluster. We estimate the value of $K$ using the cluster structure function and compare to results from  the Gap statistic, the Akaike Information Criteria (AIC) and the Bayesian Information Criteria (BIC) \cite{TK08}.  The CSF performed significantly better compared to all three alternatives, with the AIC the next closest.  The AIC was the only alternative that was competitive with the CSF for this application. Fig. \ref{resultsSynthetic} shows results for the CSF and AIC.  The good performance of the cluster structure function here follows from the optimality of Euclidean distance used to estimate $K(A)-K(x)$ as in eqn. \ref{eq.delta}.


\begin{figure}
  \includegraphics{"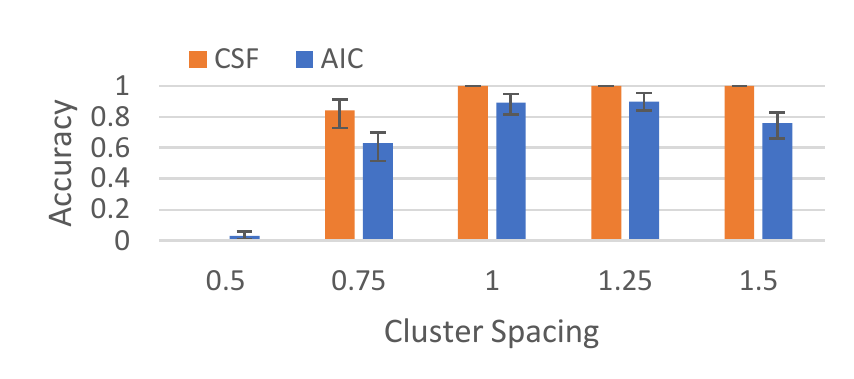"}
  \caption{\label{resultsSynthetic} Estimating the number of clusters in data generated from $K=3$ normal distributions, all with $\Sigma=[1,0;0,1]$. The distributions are located along the X axis at multiples $[0,1,2].*$ Cluster Spacing. The cluster structure function (CSF) significantly outperforms the Akaike Information Criteria (AIC). Error bars show 95\% confidence intervals from bootstrapping. 
   }
\end{figure}

\color{black}
\section{Source Code Availability}
All of the source code used to generate results in this paper is available open source from 
\url{https://git-bioimage.coe.drexel.edu/opensource/ncd}. This includes MATLAB implementations of the NCD and clustering algorithms. There is also limited support for a Python implementation, with ongoing development on that task. The ensemble segmentation algorithms are available at \url{https://leverjs.net/git}. 

\section{Acknowledgements}
Portions of this work were supported by NIH NIA (R01AG041861) and by the Human Frontiers Science Program (RGP0043/2019-203). The authors wish to thank Prof. Rafael Carazo Salas from the Univ. of Bristol UK and his group for providing sample HSC image data.

\bibliographystyle{plain}

\begin{IEEEbiography}
  [{\includegraphics[width=1in,height=1.25in,clip,keepaspectratio]
  {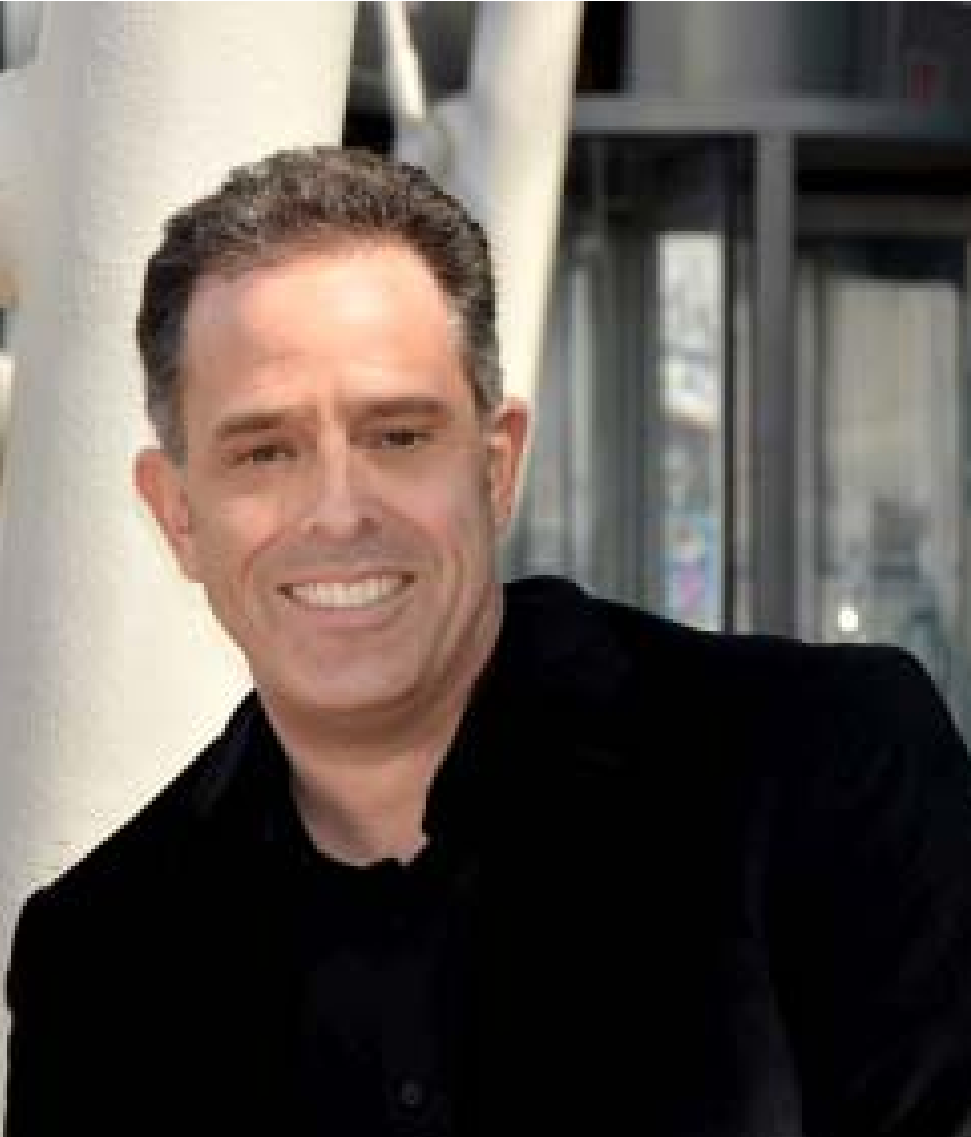}}]{Andrew R. Cohen} received his Ph.D. from the Rensselaer Polytechnic
  Institute in May 2008. He is currently an associate professor in the department of Electrical \& Computer Engineering at Drexel University. Prior to joining Drexel, he was an assistant professor in the department of Electrical Engineering and Computer Science at the University of Wisconsin, Milwaukee. He has worked as a software design engineer at Microsoft Corp. on the Windows and DirectX teams and as a CPU Product Engineer at Intel Corp. His research interests include 5-D image sequence analysis for applications in biological microscopy, algorithmic information theory, spectral methods, data visualization, and supercomputer applications. He is a senior member of the IEEE.
  \end{IEEEbiography}
  
  \begin{IEEEbiography}
  [{\includegraphics[width=1in,height=1.25in,clip,keepaspectratio]
  {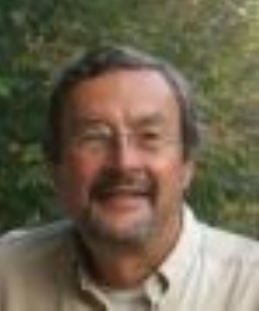}}]{Paul M.B. Vit\'anyi} received his Ph.D. from the Free University
  of Amsterdam (1978). He is a CWI Fellow at
  the national research institute for mathematics and computer
  science in the Netherlands, CWI,
  and Professor of Computer Science
  at the University of Amsterdam.  He served on the editorial boards
  of Distributed Computing, Information Processing Letters,
  Theory of Computing Systems, Parallel Processing Letters,
  International journal of Foundations of Computer Science,
  Entropy, Information,
  Journal of Computer and Systems Sciences (guest editor),
  and elsewhere. He has worked on cellular automata,
  computational complexity, distributed and parallel computing,
  machine learning and prediction, physics of computation,
  Kolmogorov complexity, information theory, quantum computing, publishing
  more than 200 research papers and some books. He received a Knighthood
  (Ridder in de Orde van de Nederlandse Leeuw) and is member of the
  Academia Europaea. Together with Ming Li
  they pioneered applications of Kolmogorov complexity
  and co-authored ``An Introduction to Kolmogorov Complexity
  and its Applications,'' Springer-Verlag, New York, 1993 (3rd Edition 2008),
  parts of which have been translated into Chinese,  Russian and Japanese.

  \end{IEEEbiography}

  \newpage
\renewcommand{\figurename}{Supplementary Figure}
\onecolumn
\setcounter{figure}{0} 
  \begin{figure}[H]
    \centering
    \includegraphics{"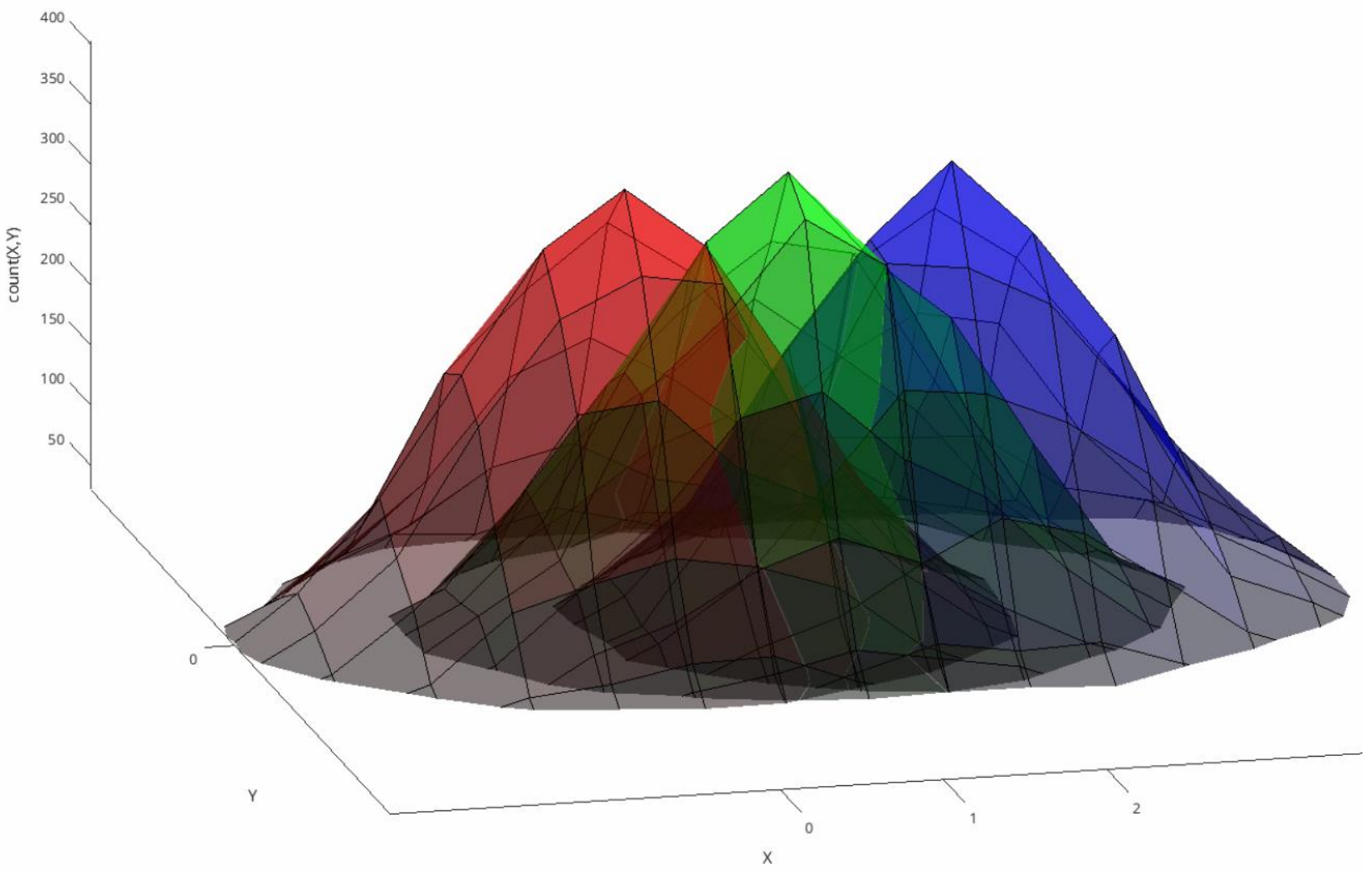"}
  \caption{\label{gtSynthetic} Histogram for synthetic dataset containing 3 clusters. Here cluster spacing equals $1.0$. Each cluster is shown in red, green, blue. The data was generated from a mixture of 3 normal distribution with means $\mu=[0,1,2]$ and identical covariance $\Sigma=[1,0;0,1]$. }
\end{figure}

\end{document}